\documentclass[accepted]{tpm2026} 
                        
\usepackage[american]{babel}

\usepackage{natbib} 
\usepackage{mathtools} 
\usepackage{booktabs} 
\usepackage{tikz} 
\usepackage{amssymb} 
\usepackage{booktabs}
\usepackage{multirow}
\usepackage{amsmath}
\usepackage{mathtools}
\usepackage{amsthm}

\theoremstyle{definition}
\newtheorem{defn}{Definition}
\theoremstyle{plain}

\newtheorem{prop}{Proposition}

\usepackage{algorithm}
\usepackage{algpseudocode}

\newcommand{\circDist}{\textsf{CW}}
\newcommand{\wassDist}{\textsf{W}}

\newcommand{\ch}{\textsf{ch}}
\newcommand{\scope}{\textsf{sc}}

\usepackage{xspace}
\usepackage{dsfont}
\newcommand{\rvars}[1]{\ensuremath{\mathbf{#1}}\xspace}
\newcommand{\X}{\rvars{X}}

\newcommand{\jstate}[1]{\ensuremath{\mathbf{#1}}\xspace}
\newcommand{\x}{\jstate{x}}
\newcommand{\y}{\jstate{y}}

\DeclareMathOperator{\Ex}{\mathbb{E}}

\newcommand{\reprop}[2]{\newtheorem*{#1}{Proposition~\ref{#1}}
\begin{#1}#2\end{#1}}

\newcommand{\method}{\textsc{PeTeR}\xspace}

\title{{\method}: Post-Training Robustification of Probabilistic Circuits}

\author[1]{\href{mailto:<acioting@asu.edu>?Subject=Your TPM 2026 paper}{Adrian Ciotinga}{}}
\author[1]{Yeming Dai}
\author[1]{YooJung Choi}
\affil[1]{%
    School of Computing and Augmented Intelligence\\
    Arizona State University University
}
  
\begin{document}
\maketitle

\begin{abstract}
    Probabilistic circuits (PCs) can model complex joint distributions while supporting exact and efficient computation of many inference queries. However, standard likelihood-based PC learning is vulnerable to overfitting and fragile generalization when confronted with data noise, small sample sizes, or distribution shifts. This can be mitigated using distributionally-robust optimization which consider worst-case distributions within a Wasserstein ball of the empirical distribution, but current methods are limited to training a model from scratch in this framework. Instead, we propose \method: a novel, data-free post-training framework designed to robustify pre-trained PCs against distribution shifts without retraining from scratch. Empirical evaluations across multiple density estimation benchmarks demonstrate that \method effectively robustifies baseline models against both random and adversarial perturbations, achieving competitive or superior performance to data-dependent robust learning baselines. 
\end{abstract}

\section{Introduction}

Distributionally robust optimization (DRO) has emerged as a principled framework for learning estimators that remain reliable under data noise, limited sample sizes, and distribution shift, by hedging against an uncertainty set of plausible distributions rather than committing to the empirical distribution alone \citep{delage2010distributionally,bertsimas2019robust}. A common approach is to define this uncertainty set as a Wasserstein ball around the empirical distribution, yielding a minimax formulation in which the estimator is optimized against the worst-case distribution, according to the appropriate performance measure, within some radius $\epsilon$. 
Although widely studied in the context of deep generative models \citep{10.1109/TKDE.2022.3224056, kuhn2019wasserstein, bauso2017distributionally}, its application to tractable probabilistic models (TPMs) remains scarce; to the best of our knowledge, \cite{pmlr-v180-peddi22a} is the sole work addressing this intersection.
While effective, existing approaches to robust learning typically focus on training a model from scratch on a dataset, tying robustness to the availability of training data and requiring that any existing models be completely replaced rather than improved.

We argue that TPMs, and probabilistic circuits (PCs) in particular, are uniquely positioned for an alternative, more flexible framework of \textit{post-training robustification}. 
Rather than learning a new model from scratch or fine-tuning it with additional data as classic DRO settings, post-training robustification takes a pre-trained model and grants it the same distributionally-robust guarantees within an $\epsilon$-Wasserstein ball \textit{without the use of any training data or complete retraining of the model.}
As we show in this paper, circuit structural properties enabling tractable inference allow us to not only compute the distance between two distributions, but also to explicitly represent and optimize against the worst-case adversarial distribution efficiently. 
Based on this insight, we propose \method, a post-training robustification framework for PCs that hedges against worst-case perturbations within a Wasserstein ball, bounded using the recently introduced Circuit-Wasserstein distance \citep{ciotinga2025optimal}. Our method reduces the optimization problem to an unconstrained problem and solves it via the gradient ascent-descent algorithm, leveraging tractable gradient computation for PCs.
We empirically demonstrate the efficacy of our approach in comparison to an existing robust MLE baseline for PCs that requires a training dataset, as---to the best of our knowledge---no existing methods are fully data-free.

\section{Preliminaries}

We use capital letters ($X$) to denote random variables and lowercase letters ($x$) to denote their assignments. Boldface denotes a set of random variables and their assignments respectively (e.g., $\X$ and $\x$).

\paragraph{Probabilistic Circuits}
%
    A probabilistic circuit (PC) $C$ over a set of 
    variables $\X$ is a parameterized, rooted directed acyclic graph (DAG) with three types of nodes: sum, product, and input nodes. 
    Each input node $n$ is associated with function $f_n$ that encodes a probability distribution over a variable $X_i \in \X$, also called its \emph{scope} $\scope(n)$. The set of child nodes for an internal node (sum or product) $n$ is denoted $\ch(n)$, and its scope is given by $\scope(n)=\bigcup_{c \in \ch(n)}\scope(c)$.
    Each sum node $n$ has normalized parameters $\theta_{n,c}$ for each child node $c$.
    For an assignment $\x$, let $\x_n$ denote its projection onto the scope $\scope(n)$ of node $n$. Then a PC rooted at node $n$ recursively defines a probability distribution $p_n(\x)$ as follows: 
    \begin{align*}
        p_n(\x) =\begin{cases}
        f_n(\x) & \text{if $n$ is an input node,} \\
      \prod_{c \in \ch(n)} p_c(\x_c) & \text{if $n$ is a product node,} \\
      \sum_{c \in \ch(n)}\theta_{n,c} p_c(\x_c) & \text{if $n$ is a sum node.}
    \end{cases}
    \end{align*}
%
Probabilistic circuits admit exact and efficient computation of many probabilistic inference queries, enabled by enforcing certain structural constraints. 
In particular, throughout this paper we assume two properties, \textit{smoothness} and \textit{decomposability}, which enable tractable (polytime) computation of marginal and conditional queries. We further state when PCs are \textit{structured-decomposable} or \textit{compatible}; see Appendix~\ref{appendix:pcdefs} for the formal definitions of these properties.

\paragraph{Robust Maximum-Likelihood Estimation}

In many machine learning paradigms, one assumes access to samples drawn from an underlying but unknown distribution $\mathcal{P}(\X)$. In the limit, the empirical distribution $\hat{P}(\X)$ approximates $\mathcal{P}(\X)$; however, assuming that $\hat{P} \approx \mathcal{P}$ is not always principled. If the observed data is contaminated by noise, limited to a small sample size, or subject to distribution shifts, optimizing directly over the empirical distribution can lead to severe overfitting and fragile generalization. To mitigate this, one can adopt a distributionally robust optimization framework~\citep{delage2010distributionally,bertsimas2019robust} that considers an uncertainty set of candidate distributions around $\hat{P}$. 
A popular choice to define this neighborhood is the Wasserstein distance, which safeguards the learned estimator $P$ by optimizing against the worst-case distribution $Q$ within an $\epsilon$-Wasserstein ball around $\hat{P}$, yielding the following minimax formulation:
\begin{align*}
    \max_{P} \min_Q & \Ex_Q [\log P(\X)] \\
    s.t.\, & \wassDist(\hat{P}, Q) \leq \epsilon
\end{align*}
\paragraph{The Wasserstein Distance}
Let $P$ and $Q$ be probability measures on metric space $\mathbb{R}^n$ with distance metric $d(\cdot, \cdot)$.
The \emph{$d$-Wasserstein distance} between $P$ and $Q$ is defined as:
\begin{align*}
    \wassDist_d(P,Q) = \inf_{\gamma \in \Gamma(P,Q)} \Ex_{\gamma(\x,\y)}[d(\x,\y)] 
\end{align*}
where $\Gamma(P,Q)$ denotes the set of all \emph{couplings}, i.e.\ joint distributions whose marginal distributions match $P$ and $Q$. That is, the following holds for all $\gamma \in \Gamma(P,Q)$:
\begin{align*}
    P(\x) = \int_{\mathbb{R}^n} \gamma(\x,\y) d\y, \quad Q(\y) = \int_{\mathbb{R}^n} \gamma(\x,\y) d\x. 
\end{align*}
Intuitively, guaranteeing the performance of estimator $P$ against distributions within Wasserstein distance $\epsilon$ of $\hat{P}$ also guarantees the performance of $P$ against similarly bounded corruptions, adversarial perturbations, or sampling noise.

\section{Robust Post-Training}

We introduce \emph{robust post-training}, a modification to robust MLE, which aims to take a pre-trained generative model encoding a baseline distribution and update it to be robust to distribution shifts. Specifically, we guarantee strong performance across an uncertainty set containing all distributions within a Wasserstein distance of $\epsilon$ around the baseline distribution. Similar to robust MLE, this also safeguards our model against corrupted data or adversarial perturbations of bounded magnitude. 
Note that the framework does not assume any access to data---including the model's original training data, making the task generally harder than robust MLE but also more flexible and widely applicable.

To solve this for PCs, we propose \method: a Post-Training Robustification framework that takes a pre-trained PC and updates its parameters to be robust to distribution shifts within a Wasserstein ball.
Our approach is enabled by tractability of probabilistic circuits. Specifically, while computing the Wasserstein distance is \#P-hard even for product distributions~\citep{taşkesen2022discreteoptimaltransportindependent}, the recently introduced Circuit-Wasserstein distance $\circDist$~\citep{ciotinga2025optimal} can be computed exactly and efficiently for structured-decomposable PCs. 
We introduce a differentiable approach to computing $\circDist$, which we use to formulate the constrained optimization problem within an $\epsilon$-$\circDist$ ball as an unconstrained gradient descent-ascent problem.

\subsection{The Optimization Problem}

Given a pre-trained PC $\hat{P}$, our objective is to learn a robust circuit $P_\theta$ parameterized by $\theta$ that has high (log)-likelihood w.r.t.\ any distribution within an $\epsilon$-$\circDist$ ball around $\hat{P}$:
%
\begin{align}
    \max_{P_\theta} \min_{Q_\phi} & \Ex_{Q_\phi} [\log P_\theta(\X)]  \label{eq:cw-dro} \\
    \text{s.t.\;} & \circDist(\hat{P}, {Q_\phi}) \leq \epsilon \nonumber
\end{align}
The Circuit-Wasserstein distance is a natural choice of a bounding metric because of its efficient and exact algorithm when applied to compatible circuits. Note that $\hat{P}$ being structured-decomposable implies that both $P_\theta$ and $Q_\phi$ are compatible with it, as all three circuits share the same structure.\footnote{We note that same structure is sufficient, not necessary, for efficient $\circDist$ computation; see \cite{ciotinga2025optimal} for details.} $\circDist$ acts as an upper-bound on the true Wasserstein distance and allows us to explicitly parameterize the worst-case distribution $Q_\phi(\X)$ using a PC.

%

\subsection{Gradient-based Robust Post-Training}

To solve the constrained optimization problem for robust post-training, we first rewrite it as an unconstrained problem.
\begin{prop}\label{prop:feasibility}
    The optimal solution to the following unconstrained Lagrangian optimization problem is a feasible and optimal solution to the constrained problem in Equation~\ref{eq:cw-dro}.
    \begin{align}
        \max_{P_\theta} \min_{Q_\phi} \max_{\lambda \geq 0} \, \Ex_{Q_\phi} [\log P_\theta(\X)] +\lambda(\circDist(\hat{P}, {Q_\phi})-\epsilon). \nonumber
    \end{align}
\end{prop}

Saddle-point optimization problems arise frequently in machine learning---e.g., GAN training, multi-agent reinforcement learning, etc. \citep{gans, marl-book}---and can be solved using the gradient ascent-descent algorithm. Our problem admits the following update rules with learning rates $\eta_\theta,\eta_\phi,\eta_\lambda$:
\begin{align*}
    \theta & \gets \theta + \eta_\theta \nabla_\theta \Ex_{Q_\phi} [\log P_\theta(\X)] \\
    \phi & \gets \phi - \eta_\phi [\nabla_\phi \Ex_{Q_\phi} [\log P_\theta(\X)] + \lambda \nabla_\phi \circDist(\hat{P}, Q_\phi)] \\
    \lambda & \gets \max(0, \lambda + \eta_\lambda[\circDist(\hat{P}, Q_\phi) - \epsilon])
\end{align*}
If $P_\theta$ and $Q_\phi$ are compatible and deterministic, the necessary gradients can be computed exactly and efficiently in a single forward and backward pass; however, without determinism, computing $\Ex_{Q_\phi} [\log P_\theta(\X)]$ is \#P-hard~\citep{vergari2021atlas}. To circumvent this, we estimate $\Ex_{Q_\phi} [\log P_\theta(\X)]$ by sampling from $Q_\phi$ when computing gradients w.r.t.\ $\theta$, and apply Jensen's inequality to bound $\Ex_{Q_\phi} [\log P_\theta(\X)] \leq \log \Ex_{Q_\phi} [ P_\theta(\X)]$ and update $\phi$ to instead minimize this upper-bound---which we can compute exactly thanks to the compatibility of $P_\theta$ and $Q_\phi$. Algorithm~\ref{algimp:couple-cw} shows the pseudocode of our post-training method.

\begin{algorithm}[t]
   \caption{$\method(\hat{P}, \epsilon)$: {\small returns parameters $\theta$ that make $\hat{P}$ robust to distributions within $\circDist$-ball of radius $\epsilon$}}\label{algimp:couple-cw}

\begin{algorithmic}[1]
    \State $P_\theta, Q_\phi \gets \hat{P}$ \Comment{Initialize $P_\theta,Q_\phi$}
   \Repeat
        \State $\x_1,\x_2,...,\x_n \overset{\text{i.i.d.}}{\sim} Q_\phi(\X)$ \Comment{Sample from $Q_\phi$}
        \State $\theta \gets \theta + \eta_\theta \nabla_\theta (\frac{1}{n} \sum_i \log P_\theta(\x_i)])$ \Comment{Estimate}
        \For{$k$ \textit{steps}}
            \State $\phi\!\gets\!\phi - \eta_\phi \nabla_\phi [\log \Ex_{Q_\phi} [ P_\theta(\X)]\!+\!\lambda \circDist(\hat{P}, Q_\phi)]$
            \State $\lambda\!\gets\!\max(0, \lambda + \eta_\lambda[\circDist(\hat{P}, Q_\phi) - \epsilon])$
        \EndFor
    \Until{\textit{convergence;}}
    
\end{algorithmic}
\end{algorithm}

\subsection{Differentiable Computation of $\circDist$}\label{sec:diffcw}

While \cite{ciotinga2025optimal} introduced the recursive algorithm for $\circDist$, computing its gradients is nontrivial because the standard forward pass relies on linear program solvers, which are not inherently compatible with backpropagation. To address this, we propose a differentiable framework for $\circDist$ that allows for computing gradients $\nabla_\theta \circDist(P_\theta, Q_\phi)$ and $\nabla_\phi \circDist(P_\theta, Q_\phi)$.

$\circDist$ is computed recursively by traversing circuits $P_\theta$ and $Q_\phi$, where the nature of the computation---and thus the gradient flow---at each step is dependent on the node type. Letting $n$ and $m$ be corresponding nodes in $P_\theta$ and $Q_\phi$ respectively, we have three cases:
\begin{itemize}
    \item \textbf{Sum Nodes:} The distance is formulated as an optimal transport problem between the children of $n$ and $m$, subject to the constraints imposed by the mixture weights $\theta$ and $\phi$. By applying the Envelope Theorem to the dual formulation of this linear program, we can compute the partial derivatives of the distance with respect to both the children’s distances and the mixture weights: $\frac{\partial \circDist(n,m)}{\partial \theta_i}=x^*_i$ and $\frac{\partial \circDist(n,m)}{\partial \phi_j}=y^*_j$.
    \item \textbf{Product Nodes:} The distance is defined as the unweighted sum of distances between corresponding children, leading to an identity gradient.
    \item \textbf{Input Nodes:} We assume the distance between base distributions is differentiable, allowing for efficient computation of $\nabla_\theta\wassDist(n,m)$ and $\nabla_\phi\wassDist(n,m)$. This holds for standard families such as Gaussian or categorical distributions.
\end{itemize}
A rigorous derivation of above can be found in Appendix~\ref{appendix:diffcw}.

\begin{table*}[t]
\centering
\scalebox{0.9}{
\setlength{\tabcolsep}{5pt}
\begin{tabular}{ll ccc ccc ccc}
\toprule
 &  & \multicolumn{3}{c}{$\mathcal{T}$} & \multicolumn{3}{c}{$\mathcal{T}_a$} & \multicolumn{3}{c}{$\mathcal{T}_r$} \\
\cmidrule(r){3-5} \cmidrule(lr){6-8} \cmidrule(l){9-11}
Dataset & $h$ & MLE-PC & RL-TPM & \textbf{\method} & MLE-PC & RL-TPM & \textbf{\method} & MLE-PC & RL-TPM & \textbf{\method} \\
\midrule
\multirow{3}{*}{NLTCS} & 1 & \multirow{3}{*}{\textbf{-6.09}} & -9.36 & \underline{-6.79} & -11.26 & \underline{-10.91} & \textbf{-9.31} & \underline{-8.41} {$\scriptstyle \pm 0.03$} & -10.13 {$\scriptstyle \pm 0.03$} & \textbf{-8.04} {$\scriptstyle \pm 0.02$} \\
 & 3 &  & -10.34 & \underline{-7.79} & -18.43 & \textbf{-11.14} & \underline{-11.56} & -11.46 {$\scriptstyle \pm 0.05$} & \underline{-11.22} {$\scriptstyle \pm 0.02$} & \textbf{-9.79} {$\scriptstyle \pm 0.02$} \\
 & 5 &  & -10.86 & \underline{-9.90} & -23.02 & \textbf{-10.61} & \underline{-11.75} & -13.24 {$\scriptstyle \pm 0.09$} & \underline{-11.44} {$\scriptstyle \pm 0.02$} & \textbf{-10.76} {$\scriptstyle \pm 0.01$} \\
\midrule
\multirow{3}{*}{MSNBC} & 1 & \multirow{3}{*}{\textbf{-6.43}} & -9.14 & \underline{-7.05} & -12.01 & \underline{-10.33} & \textbf{-8.45} & \underline{-8.31} {$\scriptstyle \pm 0.01$} & -9.88 {$\scriptstyle \pm 0.01$} & \textbf{-8.13} {$\scriptstyle \pm 0.00$} \\
 & 3 &  & \underline{-8.47} & -8.69 & -19.88 & \underline{-11.52} & \textbf{-10.96} & -11.29 {$\scriptstyle \pm 0.01$} & \underline{-10.17} {$\scriptstyle \pm 0.01$} & \textbf{-9.99} {$\scriptstyle \pm 0.00$} \\
 & 5 &  & \underline{-9.92} & -10.76 & -25.39 & \textbf{-11.85} & \underline{-11.95} & -13.52 {$\scriptstyle \pm 0.02$} & \textbf{-10.97} {$\scriptstyle \pm 0.00$} & \underline{-11.29} {$\scriptstyle \pm 0.00$} \\
\midrule
\multirow{3}{*}{Plants} & 1 & \multirow{3}{*}{\textbf{-14.46}} & -21.01 & \underline{-14.98} & -23.39 & \underline{-22.59} & \textbf{-19.62} & \underline{-18.78} {$\scriptstyle \pm 0.03$} & -23.65 {$\scriptstyle \pm 0.04$} & \textbf{-18.03} {$\scriptstyle \pm 0.02$} \\
 & 3 &  & -22.30 & \underline{-16.16} & -38.81 & \textbf{-23.22} & \underline{-26.69} & \underline{-26.26} {$\scriptstyle \pm 0.09$} & -28.76 {$\scriptstyle \pm 0.05$} & \textbf{-23.31} {$\scriptstyle \pm 0.05$} \\
 & 5 &  & -23.64 & \underline{-17.23} & -52.40 & \textbf{-27.54} & \underline{-33.30} & \underline{-32.36} {$\scriptstyle \pm 0.11$} & -32.63 {$\scriptstyle \pm 0.04$} & \textbf{-27.49} {$\scriptstyle \pm 0.07$} \\
\midrule
\multirow{3}{*}{Netflix} & 1 & \multirow{3}{*}{\textbf{-57.59}} & -58.82 & \underline{-57.66} & -61.13 & \underline{-60.67} & \textbf{-60.61} & \textbf{-58.27} {$\scriptstyle \pm 0.02$} & -59.37 {$\scriptstyle \pm 0.01$} & \underline{-58.32} {$\scriptstyle \pm 0.02$} \\
 & 3 &  & -59.27 & \underline{-58.09} & -67.12 & \textbf{-63.08} & \underline{-65.06} & \textbf{-59.56} {$\scriptstyle \pm 0.03$} & -60.55 {$\scriptstyle \pm 0.02$} & \underline{-59.69} {$\scriptstyle \pm 0.03$} \\
 & 5 &  & -60.10 & \underline{-59.21} & -72.17 & \textbf{-65.32} & \underline{-66.48} & \textbf{-60.78} {$\scriptstyle \pm 0.06$} & -61.88 {$\scriptstyle \pm 0.04$} & \underline{-61.28} {$\scriptstyle \pm 0.04$} \\
\midrule
\multirow{3}{*}{DNA} & 1 & \multirow{3}{*}{\textbf{-79.97}} & -82.69 & \underline{-81.05} & -87.68 & \underline{-86.67} & \textbf{-85.12} & \underline{-83.44} {$\scriptstyle \pm 0.09$} & -84.65 {$\scriptstyle \pm 0.05$} & \textbf{-83.42} {$\scriptstyle \pm 0.05$} \\
 & 3 &  & -84.31 & \underline{-82.62} & -102.88 & \underline{-93.38} & \textbf{-90.70} & -90.33 {$\scriptstyle \pm 0.15$} & \underline{-88.80} {$\scriptstyle \pm 0.05$} & \textbf{-88.07} {$\scriptstyle \pm 0.07$} \\
 & 5 &  & -85.85 & \underline{-84.23} & -117.82 & \underline{-98.24} & \textbf{-96.08} & -97.00 {$\scriptstyle \pm 0.23$} & \underline{-92.00} {$\scriptstyle \pm 0.08$} & \textbf{-91.18} {$\scriptstyle \pm 0.10$} \\
\midrule
\multirow{3}{*}{Movie} & 1 & \multirow{3}{*}{\textbf{-52.81}} & -55.02 & \underline{-54.07} & -61.93 & \underline{-61.21} & \textbf{-58.59} & \textbf{-58.53} {$\scriptstyle \pm 0.09$} & -59.77 {$\scriptstyle \pm 0.07$} & \underline{-58.99} {$\scriptstyle \pm 0.08$} \\
 & 3 &  & -57.06 & \underline{-53.38} & -79.51 & \underline{-69.61} & \textbf{-67.62} & -69.77 {$\scriptstyle \pm 0.20$} & \textbf{-68.82} {$\scriptstyle \pm 0.12$} & \underline{-68.91} {$\scriptstyle \pm 0.17$} \\
 & 5 &  & -59.34 & \underline{-56.09} & -96.75 & \textbf{-78.24} & \underline{-79.51} & -80.66 {$\scriptstyle \pm 0.16$} & \underline{-77.02} {$\scriptstyle \pm 0.11$} & \textbf{-75.07} {$\scriptstyle \pm 0.16$} \\
\midrule
\multirow{3}{*}{BBC} & 1 & \multirow{3}{*}{\textbf{-250.20}} & -250.72 & \underline{-250.30} & -258.72 & \textbf{-254.99} & \underline{-256.75} & \textbf{-253.04} {$\scriptstyle \pm 0.12$} & -253.43 {$\scriptstyle \pm 0.10$} & \underline{-253.10} {$\scriptstyle \pm 0.11$} \\
 & 3 &  & -251.45 & \underline{-250.55} & -275.07 & \textbf{-262.65} & \underline{-266.11} & \textbf{-258.74} {$\scriptstyle \pm 0.15$} & -259.25 {$\scriptstyle \pm 0.12$} & \underline{-258.82} {$\scriptstyle \pm 0.13$} \\
 & 5 &  & -252.52 & \underline{-250.58} & -290.24 & \textbf{-269.24} & \underline{-274.58} & \underline{-264.30} {$\scriptstyle \pm 0.30$} & -265.01 {$\scriptstyle \pm 0.26$} & \textbf{-264.16} {$\scriptstyle \pm 0.29$} \\
\bottomrule
\end{tabular}}
\caption{Test-set log-likelihood. $\epsilon \in \{1,3,5\}$: corruption budget for test-set generation. $\mathcal{T}$: original test set; $\mathcal{T}_a$: adversarially-perturbed test set; $\mathcal{T}_r$: average over 10 randomly-perturbed test sets. Boldface indicates best approach, underline second best.
}\label{tab:results}
\end{table*}

\section{Experiments}
In this section, we empirically evaluate \method on seven of the density estimation benchmark datasets \citep{lowd2010learning,van2012markov,bekker2015tractable,larochelle2011neural}, a set of binary datasets ranging from 16 to 1000 variables. We use Hidden Chow-Liu Tree (HCLT)~\citep{hclt} implemented with PyJuice~\citep{liu2024scalingtractableprobabilisticcircuits} to learn the PC structure for each dataset and estimate the parameters using \method and two other baselines for comparison.

\paragraph{Baselines} We compare against the robust MLE method (RL-TPM) proposed by \citet{pmlr-v180-peddi22a} and standard maximum-likelihood estimation via expectation-maximization (MLE-PC)~\citep{circuitem}. 
RL-TPM, designed for binary datasets, poses robust MLE as a minimax game between the PC estimator maximizing likelihood and an adversary corrupting bits with a fixed budget to minimize the estimator's likelihood. Its corruption budget $\epsilon$ corresponds to the maximum Hamming distance a data point may be moved, subsequently bounding the Wasserstein distance between the original and corrupted datasets by $\epsilon$. 
Lastly, we use MLE-PC as our base distribution $\hat{P}$.

\paragraph{Evaluation Sets} For each dataset, we construct three test sets to evaluate each method: (1) the original test set $\mathcal{T}$; (2) an adversarially-perturbed test set $\mathcal{T}_a$, where each data point receives $\epsilon$ corruptions applied greedily to maximize the drop in likelihood of the MLE-PC; and (3) $\mathcal{T}_r$ generated by randomly corrupting each data point $\epsilon$ times.

\paragraph{Results}
Table \ref{tab:results} summarizes the results. As we expect, the distributionally robust methods achieve slightly worse likelihood than MLE-PC given the unperturbed test set $\mathcal{T}$; nevertheless, we observe that \method consistently outperforms RL-TPM on $\mathcal{T}$. \method also consistently outperforms RL-TPM on the randomly-perturbed test set $\mathcal{T}_r$, which aims to capture an ``average-case'' corruption rather than a single worst case w.r.t.\ MLE-PC. Lastly, \method outperforms RL-TPM on the majority of the adversarially-perturbed datasets. We note that instances where \method is outperformed tend to be at higher corruption levels, likely due to the $\circDist$-distance being a looser upper bound on the $\wassDist$-distance than the Hamming distance used by RL-TPM is; this results in the true Wasserstein ball that \method optimizes for being smaller than the one RL-TPM does, so the true worst-case distribution---typically lying on the edge of the ball---may be outside of the $\epsilon$-$\circDist$ ball. Despite the strict $\epsilon$-$\circDist$ constraint, \method remains highly stable during training, as discussed further in Appendix~\ref{appendix:feas}.


\section{Conclusion}

We introduced \method, a data-free post-training framework that robustifies probabilistic circuits against distribution shifts within a Wasserstein ball without requiring access to any training data or retraining a circuit from scratch. By leveraging the Circuit-Wasserstein distance and deriving gradients for its computation, we formulated robust post-training within a Wasserstein ball as an unconstrained gradient descent-ascent problem that can be solved efficiently for structured-decomposable PCs. Empirically, \method consistently outperforms the data-dependent baseline across both unperturbed and perturbed test sets, particularly at lower corruption budgets, while maintaining stable training despite the non-smooth nature of Circuit-Wasserstein gradients. These results suggest that post-training robustification offers a practical and flexible alternative to training robust models from scratch and motivate further development of optimal transport-based losses for PCs. Future work includes extending \method to decomposable but non-structured PCs.

\bibliography{tpm2026-template}

\newpage

\onecolumn

\title{\method: Post-Training Robustification of Probabilistic Circuits\\(Supplementary Material)}
\maketitle

\appendix
\section{Probabilistic Circuits Background}\label{appendix:pcdefs}

This section formally establishes the core structural properties of probabilistic circuits that dictate their tractability for various probabilistic inference queries.

\begin{defn}
    A probabilistic circuit $C$ is \textbf{smooth} if every sum node $n \in C$ has the same scope as its children: $\forall n_i \in \ch(n)$, $\scope(n_i)=\scope(n)$.
\end{defn}

\begin{defn}
    A probabilistic circuit $C$ is \textbf{decomposable} if the children of every product node $n \in C$ have disjoint scope: $\forall\, n_i \neq n_j \in \ch(n)$, $\scope(n_i) \bigcap \scope(n_j) = \emptyset$.
\end{defn}

\begin{defn}
    A probabilistic circuit $C$ is \textbf{structured-decomposable} if it is decomposable and any two product nodes with the same scope partition their scope identically: $\forall\, n, m \in P$, $\scope(n) = \scope(m) \implies \{\scope(n_i) \mid n_i \in \ch(n)\} = \{\scope(m_i) \mid m_i \in \ch(m)\}$.
\end{defn}

\begin{defn}
    Two probabilistic circuits $C_1$ and $C_2$ are \textbf{compatible} if they are both structured-decomposable and any product node in $C_1$ partitions its scope identically to any product node in $C_2$ with the same scope: $\forall\, n \in {C_1}, m \in {C_2}$, $\scope(n) = \scope(m) \implies \{\scope(n_i) \mid n_i \in \ch(n)\} = \{\scope(m_i) \mid m_i \in \ch(m)\}$.

    Nodes $n$ and $m$ in compatible PCs $C_1$ and $C_2$ that share the same type (e.g., sum, product, or input) and same scope are called \textbf{corresponding} nodes.
\end{defn}

\section{Differentiable Computation of $\circDist$}\label{appendix:diffcw}

\cite{ciotinga2025optimal} detailed the recursive algorithm for computing the Circuit-Wasserstein distance, but the use of linear program solvers during the forward pass prevents the use of classic automatic differentiation techniques to compute gradients. Here, we derive the gradients for each step of the Circuit-Wasserstein algorithm. $\circDist$ looks at pairs of corresponding nodes $n$ and $m$ in $P_\theta$ and $Q_\phi$ respectively; let $n_i$ (ext. $m_j$) denote the $i$'th child of $n$ (ext. $m$). There are three cases:

\paragraph{$n$ and $m$ are Sum Nodes} 

Let $\theta_{i}$ and $\phi_j$ denote the weights of the $i$'th and $j$'th children of $n$ and $m$ respectively. The Circuit-Wasserstein distance between corresponding sum nodes is the solution to the discrete optimal transport problem where the distance between child $n_i$ and $m_j$ is their recursively-computed Circuit-Wasserstein distance. This is formulated as a linear programming problem:
\begin{align}\circDist(n,m) = \begin{cases} 
    \begin{array}{ll@{}ll}
    \displaystyle\min_w  & \sum_{i,j} \circDist(n_i,m_j)w_{i,j} &\\
    \text{s.t.}& \forall i,\, \sum_jw_{i,j}=\theta_i \\&\forall j,\, \sum_iw_{i,j}=\phi_j \\ & w_{i,j} \geq 0
    \end{array}\end{cases} 
\end{align}
With solution $w^*$ and applying the Envelope Theorem, we see that $\frac{\partial \circDist(n,m)}{\partial \circDist(n_i,m_j)}=w^*_{i,j}$. Now, denoting the dual variables for the $\theta_i$ and $\phi_j$ constraints as $x_i$ and $y_j$ respectively, the above linear programming problem has the equivalent dual form:
\begin{align}\circDist(n,m) = \begin{cases} 
    \begin{array}{ll@{}ll}
    \displaystyle\max_{x,y}  & \sum_{i}\theta_ix_i +\sum_j\phi_jy_j &\\
    \text{s.t.}& \forall i,j,\, x_i+y_j \leq \circDist(n_i,m_j)
    \end{array}\end{cases} 
\end{align}
Another application of the Envelope Theorem at the optimal point yields $\frac{\partial \circDist(n,m)}{\partial \theta_i}=x^*_i$ and $\frac{\partial \circDist(n,m)}{\partial \phi_j}=y^*_j$ respectively. Calculating $\nabla_\theta \circDist(n,m)$ and $\nabla_\phi \circDist(n,m)$ in this manner yields a \emph{sub-gradient}---as the solution to a linear programming problem is non-smooth at its optimal point---leading to the computed gradient vectors being non-smooth step functions. This would typically necessitate the use of e.g., a quadratic regularizer to smooth out the gradients, but we forego this for differentiating Circuit-Wasserstein due to the unique optimization landscape of the problem effectively nullifying this discontinuity.

\paragraph{$n$ and $m$ are Product Nodes} Since circuits $P_\theta$ and $Q_\phi$ are compatible, there is a bijective map between children $n_i$ and $m_i$ of $n$ and $m$ according to their scopes. The Circuit-Wasserstein distance between $n$ amd $m$ is simply an unweighted sum of the distances between corresponding children. Consequently, $\frac{\partial \circDist(n,m)}{\partial \circDist(n_i,m_i)}=1$.

\paragraph{$n$ and $m$ are Input Nodes} The Circuit-Wasserstein distance between input nodes is equal to the Wasserstein distance between their encoded distributions, which can be computed efficiently for many continuous and discrete families of distributions. By the assumption of tractability for input node distributions commonly made for probabilistic circuits, we assume that $\nabla_\theta \wassDist(n,m)$ and $\nabla_\phi \wassDist(n,m)$ can be computed efficiently. This assumption holds for many classes of distributions, included (but not limited to) categorical and Gaussian distributions.

\section{Feasibility and Stability of \method}\label{appendix:feas}
\reprop{prop:feasibility}{
   The optimal solution to the following unconstrained Lagrangian optimization problem is a feasible and optimal solution to the constrained problem in Equation~\ref{eq:cw-dro}.
    \begin{align}
        \max_{P_\theta} \min_{Q_\phi} \max_{\lambda \geq 0} \, \Ex_{Q_\phi} [\log P_\theta(\X)] +\lambda(\circDist(\hat{P}, {Q_\phi})-\epsilon). \label{eq:lagrangian}
    \end{align}}
\begin{proof}
    Let $\theta^*,\phi^*,\lambda^*$ denote an optimal solution to \ref{eq:lagrangian}. If $\circDist(\hat{P}, {Q_{\phi^*}}) > \epsilon$, then $\lambda\rightarrow \infty$ results in the innermost maximization going to positive infinity ($+\infty$). Since the outer optimization with respect to $Q_\phi$ seeks to \textit{minimize} this objective, it will strictly avoid any $Q_\phi$ that causes the inner maximum to diverge, provided there is at least one feasible solution (for instance, setting $Q_\phi = \hat{P}$ yields $\circDist(\hat{P}, \hat{P}) = 0 \leq \epsilon$, which results in a finite value). 

    Consequently, to prevent the objective from diverging to $+\infty$, the optimal distribution must satisfy $\circDist(\hat{P}, {Q_{\phi^*}}) \leq \epsilon$. This ensures that the optimal solution $Q_{\phi^*}$ strictly resides within the $\epsilon$-$\circDist$ ball. Thus, the optimal solution to the unconstrained Lagrangian optimization problem is guaranteed to be feasible with respect to the original constraint.
\end{proof}

\paragraph{Stability of $\circDist$ Gradients}\label{sec:discontinuity}
As noted in Section \ref{sec:diffcw}, the gradients computed for the linear programs in $\circDist$ computation are actually discontinuous sub-gradients that behave like step functions. This would typically make for highly instable training---manifesting as the optimizer's inability to maintain a feasible $\mathcal{Q_\phi}$---but we observed that training within a $\circDist$-ball is actually highly stable, with our adversarial distribution $Q_\phi$ maintaining negligible constraint violation with nearly zero hyperparameter tuning. We attribute this to the fact that a single sum node has multiple corresponding sum nodes in the other circuit. Thus, its parameter gradients are the sum of multiple sub-gradients with offset discontinuities, mitigating the instability.

\end{document}